\documentclass[USletter,10pt,english]{article}
\usepackage{setspace}
\usepackage[margin=1in]{geometry}
\usepackage[T1]{fontenc}
\usepackage{ebgaramond}
\usepackage{amsfonts}       
\usepackage{hyperref}
\usepackage{algorithm}
\usepackage{algorithmicx}
\usepackage{amssymb}
\usepackage{amsmath}
\usepackage{graphicx,wrapfig,times,amsthm,bm}
\usepackage{subfig}
\usepackage{threeparttable}
\usepackage[wide]{sidecap}
\usepackage{array}
\usepackage{algorithm}
\usepackage{algpseudocode}
\usepackage{lipsum,babel}
\usepackage{color}
\usepackage{setspace}
\usepackage{psfrag,epsf}

\usepackage{amsbsy}
\usepackage{mathrsfs}
\usepackage{mathtools}
\usepackage{setspace,multirow, url,  subfig, bm, titlesec, textcomp, gensymb}

\newtheorem{cor}{Corollary}

\DeclareMathOperator*{\argmax}{argmax}

\def \exp {\mathrm{exp}}

\newcommand{\Mnorm}[2]{{\left\vert\kern-0.30ex\left\vert\kern-0.30ex\left\vert #1 
		\right\vert\kern-0.30ex\right\vert\kern-0.30ex\right\vert}}
\newcommand{\Opnorm}[3]{{\left\vert\kern-0.25ex\left\vert\kern-0.25ex\left\vert #1 
		\right\vert\kern-0.25ex\right\vert\kern-0.25ex\right\vert}_{#2 \to #3}}
\newcommand{\norm}[2]{{\left\vert\kern-0.30ex\left\vert #1 
		\right\vert\kern-0.30ex\right\vert}}

\newcommand{\innerproductminconstant}[1]{\psi_0}

\newtheorem{theorem}{Theorem}

\newtheorem{lemma}{Lemma}

\newif\ifarxiv
\arxivtrue
\arxivfalse

\begin{document}
\ifarxiv
\doublespacing
\onecolumn
\else \fi
\title{~\\~\\Analysis of Thompson Sampling for Partially Observable Contextual Multi-Armed Bandits}
\author{Hongju Park and Mohamad Kazem Shirani Faradonbeh}%
\date{}
\maketitle

\thispagestyle{empty}
\thispagestyle{empty}

\begin{abstract}
	Contextual multi-armed bandits are classical models in reinforcement learning for sequential decision-making associated with individual information. A widely-used policy for bandits is Thompson Sampling, where samples from a data-driven probabilistic belief about unknown parameters are used to select the control actions. For this computationally fast algorithm, performance analyses are available under full context-observations. However, little is known for problems that contexts are not fully observed. We propose a Thompson Sampling algorithm for partially observable contextual multi-armed bandits, and establish theoretical performance guarantees. Technically, we show that the regret of the presented policy scales logarithmically with time and the number of arms, and linearly with the dimension. Further, we establish rates of learning unknown parameters, and provide illustrative numerical analyses.
\end{abstract}

\section{Introduction}
\label{sec:1}

Contextual Multi-Armed Bandits (CMAB) are canonical models in both theory and applications of Reinforcement Learning (RL). In this setting, there is a set of arms whose rewards depend on their multidimensional context vectors as well as the underlying parameter that reflects the weights of each context component. Thanks to their ability in modeling individual characteristics, CMAB models are widely used in different areas of automation and decision-making. For example, in personalized recommendation of news articles, CMAB models can raise the click rate by $12.5\%$, compared to context-free bandit algorithms~\cite{li2010contextual}. In dynamic treatment of mice with skin tumours, adopting biological factors as contexts, leads to a $50\%$ increase in life duration~\cite{durand2018contextual}. CMAB can also provide a useful framework for sequential decision-making in precision health by incorporating contexts such as location, calendar busyness, and heart-rate~\cite{tewari2017ads}. 

The existing literature on bandit models for decision-making under uncertainty goes back at least to the seminal work of Lai and Robbins~\cite{lai1985asymptotically} that introduces Upper Confidence Bound (UCB) algorithm. Broadly speaking, UCB prescribes acting based on optimism-based approximations of the unknown parameters, and is efficient in both discrete and continuous spaces~\cite{abbasi2011improved},~\cite{faradonbeh2020optimism}. Ensuing work establish logarithmic regret bounds of UCB that hold uniformly over time~\cite{auer2002finite}. The sequence of papers focusing on CMAB models and theoretical performance guarantees of associated reinforcement learning policies continues with showing that UCB algorithm appropriately addresses the exploitation-exploration trade-off~\cite{auer2002using}, followed by a finer analysis that improves dependence on dimensions~\cite{abbasi2011improved}, and regret bounds for linear payoffs~\cite{chu2011contextual}.

Another ubiquitous reinforcement learning policy that is usually faster than UCB, yet performs equally efficient, is Thompson Sampling~\cite{chapelle2011empirical},~\cite{russo2014learning}. The main idea of Thompson Sampling is to select actions based on samples drawn from a posterior distribution over unknown parameters~\cite{thompson1933likelihood}. The posterior is updated by the observed rewards, and balances exploring for better options and more accurate learning, versus exploiting the available information to maximize earning. Theoretical analyses start by a regret bound for multi-armed bandits~\cite{agrawal2012analysis}, and continues to CMAB counterparts~\cite{agrawal2013thompson}. Moreover, Thompson Sampling has favorable performances in continuous spaces~\cite{faradonbeh2020adaptive} and large-scale problems~\cite{hu2019note}.
Other variants and more discussions can be found in a recent tutorial by Russo et al.~\cite{russo2017tutorial}.

Further adaptive policies for CMAB models include greedy-type algorithms that are efficient if the context distribution satisfies some diversity conditions~\cite{raghavan2020greedy},~\cite{bastani2021mostly}. Moreover, the existing literature consists of studies on non-linear reward functions (of the contexts) under technical assumptions such as Lipschitz continuity. That includes, near-optimal regret bounds obtained by using partitioning techniques on the context and action space~\cite{slivkins2011contextual}, and utilizing non-parametric regression techniques for unknown non-linear reward functions~\cite{hu2020smooth}. Finally, multi-agent settings and those with latent structure of users' reward functions are studied, as well as approaches aiming to provide personalized recommendations for new users~\cite{maillard2014latent}, \cite{zhou2016latent}, \cite{hong2020latent}. 

In many applications, context vectors are observed in a partial, transformed, or noisy manner. For example, it includes situations that inquiring the entire feature vector is too expensive, context variables correspond to physically distant stations, data is provided by a network of sensors, or privacy considerations restrict perfect context observations~\cite{bensoussan2004stochastic}. For restricted contexts, reinforcement learning algorithms together with combinatorial search algorithms  demonstrate competitive empirical performance~\cite{bouneffouf2017context}. In presence of known side-information about unobserved parts of the contexts, ridge regression methods together with projections and UCB algorithms lead to improved efficiency ~\cite{tennenholtz2021bandits}. Another ubiquitous setting for studying control policies under partial observations is state space model \cite{roesser1975discrete,nagrath2006control,durbin2012time}. In this setting, unobserved states are estimated based on output observations using methods such as Kalman filter \cite{kalman1960new, stratonovich1959optimum,stratonovich1960application}, and captures important applications such as robot navigation \cite{howard2008state,surmann2020deep}. 

When the number of control actions is finite, CMAB models are widely used for data-driven control. However, unlike the aforementioned frameworks with partial observations, proper designs and comprehensive analyses of decision-making algorithms in contextual bandits with imperfect observations are not currently available. Accordingly, we study (a slightly modified) Thompson Sampling reinforcement learning algorithm for CMAB models with partially observable contexts. Note that because contexts are the main factors in determining the optimal arm, additional learning procedures are needed to estimate unobserved contexts, and so modifications in the algorithm are inevitable.

Under minimal assumptions, we establish theoretical performance guarantees showing that the regret (i.e., the cumulative decrease in rewards due to uncertainty) scales as the logarithm of time, the logarithm of the number of arms, and the dimension. We present an effective method for estimating unobserved contexts based on transformed noisy outputs, and use them to form the posterior belief about the unknown parameter, which determines the optimal candidate arm at every time step. Furthermore, we specify the rates at which Thompson Sampling learns the unknown parameter. To obtain the results, certain technical tools from the theory of martingales are leveraged, and novel methods are developed for precisely specifying the behavior of the posterior distribution and its effect on the efficiency of the algorithm. 

The remainder of this paper is organized as follows. In Section~\ref{sec:2}, we formulate the problem and discuss preliminary results. In Section~\ref{sec:3}, we present the reinforcement learning algorithm that utilizes Thompson Sampling for partially observable CMAB models. Theoretical analysis of the algorithm is provided in Section~\ref{sec:4}, followed by numerical illustrations in Section~\ref{sec:5}. Finally, concluding remarks and future directions are discussed in Section~\ref{sec:6}.

The following notation will be used throughout this paper. For a matrix $A \in \mathbb{C}^{p \times q}$, $A^\top$ denotes its transpose, and the trace of $A$ is denoted by $\mathbf{tr}(A)$. For a vector $v \in \mathbb{C}^d$, we use the Euclidean norm $||v|| = (\sum_{i=1}^d |v_i|^2)^{1/2}$, and for matrices, we use the operator norm; $||A|| = \sup_{||v||=1}  ||Av||$. Further, $\overrightarrow{u} = u/||u||$ is the unit vector indicating the direction of $u$, and $C(A)$ denotes the column space of the matrix $A$. Finally, the sigma-field generated by random vectors $\{X_1, . . . , X_n\}$ is denoted by $\sigma(X_1, \dots, X_n)$.

\section{Problem statement}
\label{sec:2}

We consider the following partially observed contextual multi-armed bandit (POCMAB) problem. Suppose that a slot machine with $N$ arms is given, and each arm $i \in \left\{ 1, \cdots, N\right\}$ has the unobserved $d$-dimensional context $x_i(t)$, which is generated independently from $N(0_d,\Sigma_x)$, where $\Sigma_x$ is the covariance matrix of $x_i(t)$. These contexts determine the rewards: At each time step $t=1,2,\dots$, the arm $a(t)$ is selected, which generates the reward $r_{a(t)}(t) = x_{a(t)}(t)^\top \mu_* + \varepsilon_{r_{a(t)}}(t)$, where $x_{a(t)}(t)$ is the context of the selected arm, $\mu_*$ is the unknown true parameter, and $\varepsilon_{r_{a(t)}}(t)$ is the reward observation noise with the distribution $N(0,\sigma^2)$. The observations at time $t$ consist of the output vectors $\{y_i(t)\}_{1\leq i \leq N}$, generated according to $y_i(t)=Ax_i(t)+ \varepsilon_{y_i}(t)$, where $\varepsilon_{y_i}(t)$ is the output observation noise that has the distribution $N(0_d, \Sigma_y)$ and $\Sigma_y$ is the covariance matrix of $y_i(t)$ given $x_i(t)$. Further, the matrix $A \in \mathbb{R}^{d \times d}$ captures the relationship between the output and the context. For the ease of presentation, we assume that $A$ is a known non-singular square matrix.

The goal is to design a reinforcement learning policy to select an arm at every time step, such that the expected reward is maximized, based on the information available at the time. That is, at time $t$, the goal is to find the optimal arm $a^*(t)=\argmax_{1\leq i \leq N} \mathbb{E}[r_{i}(t)|y_i(t)]$. The data available at time $t$, based on which we want to select $a^*(t)$, consists of the outputs $\mathbf{y_t} = \{y_i(\tau)\}_{1 \leq i \leq N,~1\leq \tau \leq t}$, the rewards of the arms selected so far $\mathbf{r_{t-1}} = \{r_{a(\tau)}(\tau)\}_{ 1 \leq \tau \leq t-1 }$, and the previously selected arms $\mathbf{a_{t-1}} = \{a(\tau)\}_{1 \leq \tau \leq t-1 }$. Note that since the context vectors $x_i(t)$ are not observed, the optimal arm $a^*(t)$ must be chosen according to a context estimate $\widehat{x}_i(t)$, based on the  observations $\{y_i(t):i=1,\dots,N\}$. It is easy to see that it suffices to select 
\begin{eqnarray}
a^*(t) = \underset{1 \leq i \leq N}{\arg\max}~ \widehat{x}_i(t)^{\top} \mu_* \label{eqa*t},
\end{eqnarray}
where $\widehat{x}_i(t)$ is the conditional expectation of $x_i(t)$ given $y_i(t)$ (the output observation of the $i$th arm at time $t$).

Due to uncertainty about the true parameter $\mu_*$, a reinforcement learning algorithm incurs a performance degradation compared to the optimal policy that knows the true parameter $\mu_*$, and selects the optimal arms $\{a^*(t)\}_{t\geq 1,}$ at every time step. Accordingly, the performance of reinforcement learning algorithms is commonly assessed by the cumulative decrease in rewards, which is called regret, and is defined as
\begin{flalign}
\mathrm{Regret}(T) &= \mathbb{E}\left[\sum_{t=1}^T r_{a^*(t)}(t)  - r_{a(t)}(t) \right].\label{eq:reg0}
\end{flalign}
Above $a(t)$ is the arm selected by the reinforcement learning policy under study. In the sequel, we present the Thompson Sampling algorithm for POCMAB models (Algorithm~\ref{Algo}), and establish a regret bound for that based on $d,N,T$.

\section{Reinforcement Learning Algorithm}
\label{sec:3}
Now, we explain a reinforcement learning algorithm that leverages Thompson Sampling to learn to maximize the reward in the POCMAB problem above, based on the output data at the time. At a high level, the main idea of the algorithm is that we maximize the expected value of the reward $r_i(t)$ given the output $y_i(t)$, because the contexts $\{x_i(t)\}_{1\leq i\leq N}$ are not observed.
To do so, using conditional expectation with respect to the observations, the regret in \eqref{eq:reg0} can be written as 
\begin{eqnarray}
\mathrm{Regret}(T) = \mathbb{E}\left[\sum_{t=1}^T \mathbb{E}\left[\left. r_{a^*(t)}(t)  - r_{a(t)}(t) \right|\{y_i(t)\}_{1\leq i \leq N} \right]\right].~~~~~~
\label{eq:reg1}
\end{eqnarray}
Note that depending on the problem understudy, technically different definitions of regret are considered in the literature \cite{bubeck2012regret}. 
The objective of the proposed reinforcement learning algorithm is to choose the arm $a(t)$ that minimizes the conditional expected reward gap given the observations $\{y_i(t)\}_{1\leq i \leq N,}$;
\begin{eqnarray}
\mathbb{E}\left[\left.\sum_{t=1}^T r_{a^*(t)}(t)  - r_{a(t)}(t) \right|\{y_i(t)\}_{1\leq i \leq N} \right] \label{eq:reg2},
\end{eqnarray}
at each time $t$, and thereby aims to minimize the regret in \eqref{eq:reg0}.

Technically, to find $a(t)$ minimizing the conditional expected reward gap in \eqref{eq:reg2}, we use the conditional distribution of the reward $r_i(t)$ given $y_i(t)$, which is derived in Appendix. The conditional distribution of $r_i(t)$ given $y_i(t)$ is
\begin{eqnarray}
    N \left( (D y_i(t))^\top \mu_* ,~\mu_*^\top(A^{\top} \Sigma_y^{-1} A + \Sigma_x^{-1})^{-1}  \mu_* + \sigma^2 \right)\label{eq:cdry},
\end{eqnarray}
where $D = (A^{\top} \Sigma_y^{-1} A + \Sigma_x^{-1})^{-1}A^\top \Sigma_y^{-1}$ is a matrix reflecting the average effect of $y_i(t)$ on $r_i(t)$. Next, let
\begin{eqnarray}
\widehat{x}_{i}(t) &=& (A^{\top} \Sigma_y^{-1} A + \Sigma_x^{-1})^{-1}A^\top \Sigma_y^{-1} y_{i}(t) = Dy_{i}(t) \label{eq:xhat}.
\end{eqnarray}
In fact, $\widehat{x}_{i}(t)$ is the conditional expectation $\mathbb{E}[x_i(t)|y_i(t)]$. 
Putting \eqref{eq:cdry} and \eqref{eq:xhat} together, the conditional expected reward gap in \eqref{eq:reg2} can be written as
\begin{eqnarray}
\mathbb{E}\left[\left. r_{a^*(t)}(t)  - r_{a(t)}(t) \right|\{y_i(t)\}_{1\leq i \leq N} \right] &=& \mathbb{E}\left[\left.\mathbb{E}\left[ r_{a^*(t)}(t)  - r_{a(t)}(t)|x_i(t)\right] \right|\{y_i(t)\}_{1\leq i \leq N} \right] \nonumber\\
&=& \mathbb{E}\left[\left. (x_{a^*(t)}(t) - x_{a(t)}(t))^\top \mu_* \right|\{y_i(t)\}_{1\leq i \leq N} \right] \nonumber\\
&=& (\widehat{x}_{a^*(t)}(t) - \widehat{x}_{a(t)}(t))^\top \mu_*. \label{eq:reg3}
\end{eqnarray}
Thus, a policy is designed to choose the arm maximizing $\widehat{x}_{i}(t)^\top \mu_*$. To ensure that the algorithm performs enough exploration, we use the sample $\widetilde{\mu}(t)$ from the posterior distribution
\begin{eqnarray}
N(\widehat{\mu}(t), B(t)^{-1}), \label{eq:pos}
\end{eqnarray}
where the posterior mean $\widehat{\mu}(t)$ and the inverse of the covariance matrix $B(t)$ are as follows:
\begin{eqnarray}
B(t) &=& \Sigma^{-1} + \sum_{\tau=1}^{t-1} \widehat{x}_{a(\tau)}(\tau)\widehat{x}_{a(\tau)}(\tau)^\top,\label{eq:B2}\\
\widehat{\mu}(t) &=& B(t)^{-1} \sum_{\tau=1}^{t-1}\widehat{x}_{a(\tau)}(\tau)r_{a(\tau)}(\tau) \label{eq:muhat2}.
\end{eqnarray}
Based on the estimates of the contexts and the sample $\widetilde{\mu}(t)$, we select $a(t)$ such that 
\begin{eqnarray}
a(t)= \underset{1 \leq i \leq N}{\arg\max}~ \widehat{x}_i(t)^{\top} \widetilde{\mu}(t) \label{eq:at}.
\end{eqnarray}
Then, we observe the reward $r_{a(t)}(t)$ of the arm $a(t)$, and update the posterior according to
\begin{eqnarray}
B(t+1) &=& B(t) + \widehat{x}_{a(t)}(t) \widehat{x}_{a(t)}(t)^{\top}, \label{eq:B}\\
\widehat{\mu}(t+1) &=& B(t+1)^{-1}(B(t)\widehat{\mu}(t) + \widehat{x}_{a(t)}(t) r_{a(t)}(t)).\:\:\:\:~~\label{eq:muhat}
\end{eqnarray}
The initial values are $\widehat{\mu}(1) = 0_{d}$ and $B(1) = \Sigma^{-1}$, where $\Sigma$ is an arbitrary symmetric positive definite matrix.

\begin{algorithm}[] 
	\begin{algorithmic}[1]
		\State Set $B(1) = \Sigma^{-1}$, $\widehat{\mu}(1) = \mathbf{0}_d$
		\For{$t = 1,2, \dots, $} 
		\For{$i=1,\dots,N$}
		\State Estimate context by $\widehat{x}_i(t)$ in \eqref{eq:xhat} 
		\EndFor
		\State Sample $\widetilde{\mu}(t)$ from $N(\widehat{\mu}(t),B(t)^{-1})$
		\State Select arm $a(t)= \underset{1 \leq i \leq N}{\arg\max}~ \widehat{x}_i(t)^{\top} \widetilde{\mu}(t)$
		\State Gain reward $r_{a(t)}(t) = x_{a(t)}(t)^\top \mu_* + \epsilon_{r_{a(t)}}(t)$
		\State Update $B(t+1)$ and $\widehat{\mu}(t+1)$ by \eqref{eq:B} and  \eqref{eq:muhat}
		\EndFor
	\end{algorithmic}
\caption{: Thomson Sampling RL policy for POCMAB}  \label{Algo}
\end{algorithm}
The pseudo-code of Thompson sampling for POCMAB is provided in Algorithm~\ref{Algo}. At every time and for each arm, Algorithm~\ref{Algo} calculates the context estimate $\widehat{x}_i(t)$ according to \eqref{eq:xhat}. Then, it chooses the arm $a(t)$ by \eqref{eq:at}, based on $\widetilde{\mu}(t)$ generated from the posterior in \eqref{eq:pos}, and updates $\widehat{\mu}(t)$ and $B(t)$ according to \eqref{eq:B} and \eqref{eq:muhat}. So, Algorithm~\ref{Algo} selects the arm maximizing $\widehat{x}_i(t)^\top \widetilde{\mu}(t)$ as a reliable estimate of the unknown expected reward at time $t$.
\section{Analysis of Algorithm~\ref{Algo}}
\label{sec:4}
In this section, we provide theoretical performance guarantees for the reinforcement learning policy in Algorithm~\ref{Algo}, establishing that it efficiently learns optimal decisions from the data of partial observations. In the first result we show that Algorithm~\ref{Algo} learns the unknown parameter $\mu_*$, fast and accurately. Then, in Theorem~\ref{thm:2}, we provide regret analysis, indicating that the regret of Algorithm~\ref{Algo} scales logarithmically with both the number of arms $N$, as well as the time of interaction with the environment $T$, and scales linearly with the dimension $d$.

The following result shows that $\widehat{\mu}(t)$ is a consistent estimator and its covariance matrix shrinks proportional to the inverse of the time of interacting with the environment in Algorithm~\ref{Algo}. Therefore, Theorem~\ref{thm:1} provides sample efficiency for the Thompson Sampling reinforcement learning policy for POCMAB in Algorithm~\ref{Algo}. 

\begin{theorem}
	In Algorithm \ref{Algo}, let $\widehat{\mu}(t)$ be the parameter estimate at time $t$, defined by \eqref{eq:muhat}. Then, we have $\underset{t \rightarrow \infty}{\lim} \widehat{\mu}(t)=\mu_*$, as well as $\mathrm{Cov}\left(\widehat{\mu}(t) \right) = O(t^{-1})$.
	\label{thm:1}
\end{theorem}
\begin{proof}

First, for the prior $N(0_d,\Sigma)$ of $\mu_*$, \eqref{eq:B2} and \eqref{eq:muhat2} imply that
\begin{eqnarray}
\mathbb{E}\left[\widehat{\mu}(t)\right]=\mathbb{E}\left[ B(t)^{-1} \sum_{\tau=1}^{t-1} \widehat{x}_{a(\tau)}(\tau) \widehat{x}_{a(\tau)}(\tau)^{\top}\mu_*\right]
= (I_d - \mathbb{E}[B(t)^{-1}] \Sigma^{-1})\mu_*.\label{eq:emuhat}
\end{eqnarray}
Further, let $\mathscr{F}_{t} = \sigma \left\{ \{y_i(\tau)\}_{1 \leq i \leq N,~1 \leq \tau \leq t } , \{a(\tau)\}_{1\leq \tau \leq t} \right\}$ be the sigma-field generated by the sequence of all observations and actions by time $t$. Given the sigma-field $\mathscr{F}_{t-1}$, we have
\begin{eqnarray}
\mathbb{E}\left[\widehat{\mu}(t)|\mathscr{F}_{t-1}\right]
&=& \mathbb{E}\left[ \left. B(t)^{-1} \sum_{\tau=1}^{t-1} \widehat{x}_{a(\tau)}(\tau) \widehat{x}_{a(\tau)}(\tau)^{\top}\mu_*\right|\mathscr{F}_{t-1}\right]= (I_d -  B(t)^{-1}\Sigma^{-1})\mu_*,\label{eq:emuhatf}\\
\mathrm{Cov}\left(\widehat{\mu}(t)|\mathscr{F}_{t-1}\right)
&=& B(t)^{-1}\left(\sum_{\tau=1}^t \mathrm{Var}\left(r_{a(\tau)}(\tau)|\mathscr{F}_{t-1}\right)  \widehat{x}_{a(\tau)}(\tau) \widehat{x}_{a(\tau)}(\tau)^{\top}\right)B(t)^{-1}\nonumber\\ 
&=&
B(t)^{-1} (B(t) -  \Sigma^{-1}) B(t)^{-1} \sigma^2_{ry}, \label{eq:covmuhatf}
\end{eqnarray}
where $\sigma^2_{ry} = \mathrm{Var}(r_i(t)|y_i(t)) = \mu_*^\top (A^\top \Sigma_y^{-1} A + \Sigma_x^{-1})^{-1} \mu_* + \sigma^2$ is derived in Appendix.  Using \eqref{eq:emuhat}, \eqref{eq:emuhatf} and \eqref{eq:covmuhatf}, we obtain 
\begin{flalign}
\mathrm{Cov}(\widehat{\mu}(t)) 
&=\mathrm{Cov}(\mathbb{E}[\widehat{\mu}(t)|\mathscr{F}_{t-1} ]) + \mathbb{E}[\mathrm{Cov} (\widehat{\mu}(t)|\mathscr{F}_{t-1})]\nonumber\\
&= \mathbb{E}\left[B(t)^{-1} \Sigma^{-1} \mu_*\mu_*^\top \Sigma^{-1} B(t)^{-1} \right]- \mathbb{E} \left[B(t)^{-1} \right] \Sigma^{-1}\mu_*\mu_*^\top \Sigma^{-1}\mathbb{E} \left[B(t)^{-1} \right] \nonumber\\
&+ \mathbb{E}\left[B(t)^{-1}\right]\sigma^2_{ry}- \mathbb{E}\left[B(t)^{-1}\Sigma^{-1} B(t)^{-1}\right]\sigma^2_{ry}
\label{eq:covmuhat}.
\end{flalign}
Next, we show that $\underset{t\rightarrow \infty}{\lim} t^{-1}B(t)$ is a positive definite matrix. It implies that $\mathrm{Cov}(\widehat{\mu}(t)) = O(t^{-1})$, since the other terms in \eqref{eq:covmuhat} are $O(t^{-2})$, except $\mathbb{E}\left[B(t)^{-1}\right]\sigma^2_{ry}$. For this purpose, let $S= (D \Sigma_y D^\top)^{1/2}$, and define
\begin{eqnarray}
X_t &=&  \sum_{\tau=1}^t \left( S^{-1}\widehat{x}_{a(\tau)}(\tau)\widehat{x}_{a(\tau)}(\tau)^{\top}S^{-1} - \mathbb{E}[S^{-1}\widehat{x}_{a(\tau)}(\tau)\widehat{x}_{a(\tau)}(\tau)^{\top}S^{-1}|\mathscr{F}_{\tau-1}]\right),\nonumber\\
Y_t &=& \sum_{\tau=1}^t \tau^{-1}(X_{\tau} - X_{{\tau}-1}).\nonumber
\end{eqnarray}
Then, $X_t$ and $Y_t$ are matrix valued martingales adapted to the filtration $\{\mathscr{F}_{t}\}_{t\geq 1}$. To see that, observe that the following two equivalences
\begin{eqnarray}
\mathbb{E}\left[\left.S^{-1}\widehat{x}_{a(\tau)}(\tau)\widehat{x}_{a(\tau)}(\tau)^{\top}S^{-1}\right|\mathscr{F}_{t-1}\right] &=& S^{-1}\widehat{x}_{a(\tau)}(\tau)\widehat{x}_{a(\tau)}(\tau)^{\top}S^{-1},\label{eq:x0}\\
\mathbb{E}\left[\mathbb{E}\left[\left.\left.S^{-1}\widehat{x}_{a(\tau)}(\tau)\widehat{x}_{a(\tau)}(\tau)^{\top}S^{-1}\right|\mathscr{F}_{\tau-1}\right]\right|\mathscr{F}_{t-1}\right] 
&=& \mathbb{E}\left[\left.S^{-1}\widehat{x}_{a(\tau)}(\tau)\widehat{x}_{a(\tau)}(\tau)^{\top}S^{-1}\right|\mathscr{F}_{\tau-1}\right]\label{eq:x1}.
\end{eqnarray}
lead to 
\begin{eqnarray}
\mathbb{E}\left[X_t|\mathscr{F}_{t-1}\right] 
&=& \sum_{\tau=1}^t \left( \mathbb{E}\left[\left.S^{-1}\widehat{x}_{a(\tau)}(\tau)\widehat{x}_{a(\tau)}(\tau)^{\top}S^{-1}\right|\mathscr{F}_{t-1}\right] - \mathbb{E}\left[\mathbb{E}\left[\left.\left.S^{-1}\widehat{x}_{a(\tau)}(\tau)\widehat{x}_{a(\tau)}(\tau)^{\top}S^{-1}\right|\mathscr{F}_{\tau-1}\right]\right|\mathscr{F}_{t-1}\right]\right)\nonumber\\
&=& \sum_{\tau=1}^t \left( \mathbb{E}[S^{-1}\widehat{x}_{a(\tau)}(\tau)\widehat{x}_{a(\tau)}(\tau)^{\top}S^{-1}|\mathscr{F}_{t-1}] - \mathbb{E}\left[\left. S^{-1}\widehat{x}_{a(\tau)}(\tau)\widehat{x}_{a(\tau)}(\tau)^{\top}S^{-1}\right|\mathscr{F}_{\tau-1}\right]\right)=X_{t-1},~~~~~~~~~\label{eq:x2}
\end{eqnarray}

for $\tau < t$ and
\begin{flalign*}
\mathbb{E}\left[\left.S^{-1}\widehat{x}_{a(\tau)}(\tau)\widehat{x}_{a(\tau)}(\tau)^{\top}S^{-1}\right|\mathscr{F}_{t-1}\right] - \mathbb{E}\left[\mathbb{E}\left[\left.\left.S^{-1}\widehat{x}_{a(\tau)}(\tau)\widehat{x}_{a(\tau)}(\tau)^{\top}S^{-1}\right|\mathscr{F}_{\tau-1}\right]\right|\mathscr{F}_{t-1}\right] =0_{d \times d},
\end{flalign*}
for $\tau = t$.
Further, since $\mathbb{E}[X_{\tau}|\mathscr{F}_{t-1}]=X_{\tau}$, for $\tau < t$, and we have $\mathbb{E}[X_t|\mathscr{F}_{t-1}] - \mathbb{E}[ X_{t-1}|\mathscr{F}_{t-1}]=0_{d \times d}$, it holds that
\begin{flalign*}
\mathbb{E}\left[Y_t|\mathscr{F}_{t-1}\right] 
= \sum_{\tau=1}^t \tau^{-1}\left(\mathbb{E}[X_{\tau}|\mathscr{F}_{t-1}] - \mathbb{E}[ X_{{\tau}-1}|\mathscr{F}_{t-1}] \right) =\sum_{\tau=1}^{t-1} \tau^{-1}\left(X_{\tau} - X_{{\tau}-1}\right) = Y_{t-1}.\nonumber
\end{flalign*}
Now, define the martingale difference sequence $Z_t = X_t-X_{t-1}$, and let $X_{tij}$ be the $ij$th entry of $X_t$, to get
\begin{flalign*}
\mathbb{E}\left[X_{tij}^2\right] = \mathbb{E}\left[\left(\sum_{\tau=1}^{t}Z_{\tau ij}\right)^2 \right] = \sum_{\tau=1}^{t} \mathbb{E}\left[Z_{\tau ij}^2\right] + 2 \sum_{\tau_1 < \tau_2} \mathbb{E}\left[  Z_{\tau_1ij} Z_{\tau_2ij} \right] = \sum_{\tau=1}^{t} \mathbb{E}\left[Z_{\tau ij}^2\right],
\end{flalign*}
using the fact that $\mathbb{E}\left[  Z_{\tau_1ij} Z_{\tau_2ij} \right] = \mathbb{E}\left[ Z_{\tau_1ij} \mathbb{E} \left[  \left.  Z_{\tau_2ij}\right| \mathscr{F}_{\tau_2-1} \right] \right]=0$  for all $\tau_1 < \tau_2$. 

Using the above, we show that $Y_t$ is a square-integrable martingale. To that end, since $\{X_t - X_{t-1}:t\geq 1\}$ is a martingale difference sequence, we have
\begin{eqnarray*}
    \mathbb{E}\left[Y_{tij}^2\right] = \sum_{\tau=1}^t \tau^{-2} \left(\mathbb{E} \left[ X_{\tau ij}^2 \right] - \mathbb{E} \left[ X_{(\tau-1)ij }^2 \right] \right) 
    = \sum_{\tau=1}^t \tau^{-2} \mathbb{E}\left[Z_{\tau ij}^2\right],
\end{eqnarray*}
where $X_0=0_{d \times d}$, and $Y_{tij}$ is the $ij$th entry of $Y_{t}$. Since $\mathbb{E}\left[Z_{\tau ij}^2\right] \leq \mathbb{E}\left[||S^{-1}\widehat{x}_{a(t)}(t)||^4\right]$, for all $\tau$, $i$, and $j$, the expectation $\mathbb{E}[Y_{tij}^2]$ is finite. So, by Martingale Convergence Theorem~\cite{doob1953stochastic}, the martingale $Y_t$  converges almost surely to a limit $Y$, such that $\mathbb{E}[|Y|] < \infty$.	It is straightforward to see that $t^{-1}X_t = Y_t - t^{-1}\sum_{\tau=1}^t Y_\tau$. Thus, since $\lim\limits_{t\rightarrow \infty} Y_t = Y$, the average of the sequence converges to the same limit as well; $\lim\limits_{t\rightarrow \infty} t^{-1}\sum_{\tau=1}^t Y_\tau = Y$. Thus, $t^{-1}X_t$ converges to $0_{d\times d}$. To show that $\lim_{t\rightarrow \infty }t^{-1}B(t)$ is a positive definite matrix, decompose $X_t$ as follows:
\begin{flalign*}
X_t = S^{-1}(B(t)-\Sigma^{-1} )S^{-1} -\sum_{\tau=1}^t  \mathbb{E}[S^{-1}\widehat{x}_{a(\tau)}(\tau)\widehat{x}_{a(\tau)}(\tau)^{\top}S^{-1}|\mathscr{F}_{\tau-1}].
\end{flalign*}
Since $\lim_{t\rightarrow \infty}t^{-1}X_t = 0_{d\times d}$, we have
\begin{flalign}
\lim_{t\rightarrow \infty} t^{-1} S^{-1}B(t)S^{-1} = \lim_{t\rightarrow \infty} t^{-1} \sum_{\tau=1}^t  \mathbb{E}[S^{-1}\widehat{x}_{a(\tau)}(\tau)\widehat{x}_{a(\tau)}(\tau)^{\top}S^{-1}|\mathscr{F}_{\tau-1}]\label{eq:lim}.
\end{flalign}
To proceed, we express the following result about the matrix $M=\underset{t \rightarrow \infty}{\lim}\mathbb{E}[ S^{-1} \widehat{x}_{a(t)}(t)\widehat{x}_{a(t)}(t)^{\top} S^{-1}|\mathscr{F}_{t-1}]$, for which the proof is deferred to Appendix. Now, by \eqref{eq:lim}, we have
\begin{eqnarray*}
    \underset{t\rightarrow \infty}{\lim} t^{-1} S^{-1}B(t) S^{-1} = M,
\end{eqnarray*}
which according to Lemma \ref{lem1} is a positive definite matrix. Finally, the latter result, together with \eqref{eq:covmuhat}, implies that 
\begin{equation*}
    \underset{t\rightarrow \infty}{\lim} t\mathrm{Cov}\left(\widehat{\mu}(t) \right) = \underset{t\rightarrow \infty}{\lim} t \mathbb{E}[B(t)^{-1}]\sigma^2_{ry} = SM^{-1}S\sigma^2_{ry},
\end{equation*}
which is the desired result.
\end{proof}

\begin{lemma}
The matrix $M=\underset{t \rightarrow \infty}{\lim}\mathbb{E}[ S^{-1} \widehat{x}_{a(t)}(t)\widehat{x}_{a(t)}(t)^{\top} S^{-1}|\mathscr{F}_{t-1}]$ is deterministic and positive definite. 
\label{lem1}
\end{lemma}

Theorem~\ref{thm:1} establishes the square-root consistency of the parameter estimate $\widehat{\mu}(t)$, indicating the Algorithm~\ref{Algo} effectively learns the unknown true parameter $\mu_*$. Here, the inverse of $\mathrm{Cov}(\widehat{\mu}(t))$ grows linearly with time $t$, only when the smallest eigenvalue of $A^\top A$ is non-zero. If $A$ is singular, the maximum eigenvalue of $\mathrm{Cov}(\widehat{\mu}(t))$ does not decrease as $t$ becomes larger. This also affects the consistency of learning the unknown parameter. A similar result holds for the samples $\widetilde{\mu}(t)$, as elaborated in the following corollary, for which the details are provided in Appendix.
\begin{cor}
\label{col:1} 
For the samples $\{\widetilde{\mu}(t)\}_{t\geq 1}$ in Algorithm~\ref{Algo}, we have
$$\lim_{t \rightarrow \infty} \widetilde{\mu}(t) = \mu_*, \:\:\:\:\:\:\:\mathrm{Cov}(\widetilde{\mu}(t)) = O(t^{-1}).$$
\end{cor}
The following result provides a regret bound, and states that Algorithm~\ref{Algo} is able to efficiently learn optimal arms in POCMAB.
\begin{theorem}
\label{thm:2}
	For the regret of Algorithm \ref{Algo}, we have 
	$$\mathrm{Regret}(T)=O\left(d \sqrt{\log N}\log T\right).$$
\end{theorem}
Before proceeding towards the proof of Theorem 2, we discuss the intuition it provides. Since the regret at time $t$ grows due to the difference between $\mu_*$ and $\widetilde{\mu}(t)$, the growth rate of regret depends on the shrinkage rate of $||\mu_*-\widetilde{\mu}(t)||^2$. According to Corollary 1, the shrinkage rate is $O(dt^{-1})$. Thus, aggregating the errors for the time period $1 \leq t \leq T$, the scaling with respect to $T$ becomes logarithmic (see \eqref{eq:log}), while the scaling with $d$ is linear. On the other hand, the regret scales logarithmically \emph{slow} with the number of arms $N$, because $N$ has two opposite effects. On the one hand, since $N$ is the total number of options, the probability of choosing a sub-optimal arms increases as $N$ grows. On the other hand, the difference between the reward of the optimal arm and that of the chosen arm becomes smaller as $N$ grows. The consequences of the two effects compensate each other, leading to the slow growth of the regret with respect to $N$. As mentioned, the suggested regret bound works for non-singular $A$. If $A$ is singular and $\mu_* \in C\left(A^\top\right)^\perp$, the regret grows linearly with $T$.

\begin{proof}
First, for the regret of Algorithm \ref{Algo}, it holds that $\mathrm{Regret}(T)=  \mathbb{E}\left[\sum_{t=1}^T  (\widehat{x}_{a^*(t)} - \widehat{x}_{a(t)})^\top \mu_* \right]$, according to \eqref{eq:reg1} and \eqref{eq:reg3}. To proceed, we show that for an arbitrary $\mu_* \in \mathbb{R}^d$, it holds that $ \mathbb{E}\left[\underset{\widehat{x}_i(t),1 \leq i \leq N}{\argmax}\left\{\widehat{x}_i(t)^\top \mu_* \right\} \right] = c_N \overrightarrow{S\mu_*}$, where the constant
\begin{eqnarray}
c_N =  \mathbb{E}\left[\underset{1\leq i \leq N}\max\{V_i:V_i\sim N(0,1)\}\right]\label{eq:cn}
\end{eqnarray}
captures the magnitude, and the unit vector $\overrightarrow{S\mu_*}$ indicates the direction of the expected value of the vector $\widehat{x}_i(t)$ that achieves the maximum value inside the expectation.

To show the above result, define
\begin{eqnarray}
Z(\mu,N) = \underset{Z_i,1 \leq i \leq N}{\argmax}\left\{ Z_i^\top \mu  \right\}\label{eq:zmun},
\end{eqnarray}
where $Z_i$ are independent standard $d$-dimensional normally distributed random vectors. The vector $Z_i$ can be decomposed as $Z_i = P_{\mu} Z_i + (I_d-P_{\mu}) Z_i$, where $P_{\mu}$ is the projection matrix onto $C(\mu)$, which is the $1$-dimensional subspace of the vectors inline with $\mu$. Then, we have $Z(\mu,N) = \underset{\widehat{x}_i(t),1 \leq i \leq N}{\argmax} \left\{ (P_{\mu} Z_i(t))^\top \mu  \right\}$, because $P_{\mu}\mu = \mu$. This implies that only the first term, $P_{\mu} Z_i$, affects the result of $\underset{Z_i,1 \leq i \leq N}{\argmax} \left\{ Z_i^\top  \mu  \right\}$. This means that $Z(\mu,N)$ has the same distribution as $P_{\mu} Z(\mu,N) +(I_d- P_{\mu})Z_{i}$, which means
\begin{eqnarray}
Z(\mu,N) \overset{d}{=}P_{\mu} Z(\mu,N) +(I_d- P_{\mu})Z_{i}\label{eq:equaldis},
\end{eqnarray}
where $\overset{d}{=}$ is used to denote equality of the probability distributions. Thus, since projection on a subspace is a linear operator, it interchanges with expectation, and so we have
\begin{eqnarray}
\mathbb{E}[Z(\mu,N)]=\mathbb{E}\left[P_{\mu} Z(\mu,N) +(I_d- P_{\mu})Z_{i}\right]=P_{\mu} \mathbb{E}[Z(\mu,N)] \in C(\mu) \label{eq:cmu}.
\end{eqnarray}

Next, we claim that $\mathbb{E}[Z(\mu,N) ] = c_N \overrightarrow{\mu}$, where $c_N$ is defined in \eqref{eq:cn}, for which it is known that \cite{cramer2016mathematical}:
\begin{eqnarray}
c_N = O\left(\sqrt{\log N}\right). \label{eq:cnlogn}
\end{eqnarray}
Because $Z_i^\top \overrightarrow{\mu}$ has the standard normal distribution $N(0, 1)$, according to \eqref{eq:cn}, we have $\mathbb{E}\left[\underset{1 \leq i \leq N}{\max} \left\{Z_i^\top \overrightarrow{\mu} \right\}\right]=c_N$. Based on the definition in \eqref{eq:zmun}, it holds that $Z(\mu,N)^\top \overrightarrow{\mu} =\underset{1 \leq i \leq N}{\max} \left\{Z_i^\top \overrightarrow{\mu} \right\}$. Moreover, because $\mathbb{E}[Z(\mu,N) ] \in C(\mu)$ by \eqref{eq:cmu}, we have $c_N = \mathbb{E}[Z(\mu,N) ]^\top \overrightarrow{\mu} = ||\mathbb{E}[Z(\mu,N) ]||~||\overrightarrow{\mu}|| = ||\mathbb{E}[Z(\mu,N) ]||$. Putting the above together, we obtain 
\begin{eqnarray}
\mathbb{E}[Z(\mu,N) ] = c_N \overrightarrow{\mu}.\label{eq:cnmu}
\end{eqnarray}
Next, we apply the result in \eqref{eq:cnmu} to $\widehat{x}_{a^*(t)}(t)$ and $\widehat{x}_{a(t)}(t)$. The definition of $Z(\mu,N)$ in \eqref{eq:zmun} implies that  $S^{-1}\widehat{x}_{a^*(t)}(t)$ can be written as
\begin{equation*}
    \underset{S^{-1}\widehat{x}_i(t),1 \leq i \leq N}{\argmax} \left\{ (S^{-1}\widehat{x}_i(t))^\top  S\mu_*  \right\} = Z(S\mu_*,N).
\end{equation*}
Similarly, it holds that
\begin{flalign*}
S^{-1}\widehat{x}_{a(t)}(t) = \underset{ S^{-1}\widehat{x}_i,1 \leq i \leq N}{\argmax} \left\{ (S^{-1}\widehat{x}_i(t) )^\top S \widetilde{\mu}(t)  \right\} = Z(S\widetilde{\mu}(t),N).
\end{flalign*}
Using \eqref{eq:cnmu}, we can find the expected values as follows:
\begin{eqnarray}
\mathbb{E}[ S^{-1} \widehat{x}_{a^*(t)}] &=&  c_N \overrightarrow{S\mu_*}\label{eq:a*tcn},\\
\mathbb{E}[S^{-1} \widehat{x}_{a(t)}|\widetilde{\mu}(t)] &=& c_N \overrightarrow{S\widetilde{\mu}(t)}.\label{eq:atcn} 
\end{eqnarray}
Using the above equations, we have 

\begin{eqnarray}
\mathbb{E}\left[\sum_{t=1}^T \left(S^{-1} \widehat{x}_{a^*(t)}(t) - S^{-1} \widehat{x}_{a(t)}(t)\right)^\top S \mu_* \right] 
&=& \mathbb{E}\left[ \mathbb{E}\left[\left. \sum_{t=1}^T \left(S^{-1} \widehat{x}_{a^*(t)}(t) - S^{-1} \widehat{x}_{a(t)}(t)\right)^\top S \mu_* \right|\widetilde{\mu}(t) \right]\right]\nonumber\\
&=&  \mathbb{E}\left[ \sum_{t=1}^T \left(c_N \overrightarrow{S\mu_*} -  c_N \overrightarrow{S\widetilde{\mu}(t)} \right)^\top S \mu_* \right]\label{eq:reg}
\end{eqnarray}
for the expected gap.

Now, let $\theta_t$ denote the angle between $S\mu_*$ and $S\widetilde{\mu}(t)$, defined as
\begin{eqnarray}
\theta_t = \cos^{-1} \frac{<S\mu_*,S\widetilde{\mu}(t)>}{||S\mu_*||~||S\widetilde{\mu}(t)||}\label{eq:thetat}~\in [0,\pi].
\end{eqnarray}
Since the vectors $\overrightarrow{S\mu_*}$ and $\overrightarrow{S\widetilde{\mu}(t)}$ are of the same length, the angle between $\overrightarrow{S\mu_*} -  \overrightarrow{S\widetilde{\mu}(t)}$ and $\overrightarrow{S\mu_*}$ is $(\pi-\theta_t)/2$, which leads to $\left|\left|\overrightarrow{S\mu_*} -  \overrightarrow{S\widetilde{\mu}(t)}\right|\right| =2\sin (\theta_t/2)$. Thus, we get
\begin{flalign*}
\left(\overrightarrow{S\mu_*} -  \overrightarrow{S\widetilde{\mu}(t)}\right)^\top S\mu_* 
&=  ||S\mu_*||\left|\left|\overrightarrow{S\mu_*} -  \overrightarrow{S\widetilde{\mu}(t)}\right|\right| \cos \left(\frac{\pi-\theta_t}{2}\right) \\
&= 2 ||S\mu_*|| \sin \left(\frac{\theta_t}{2}\right)  \cos \left(\frac{\pi-\theta_t}{2}\right) \\
&= 2 ||S\mu_*|| \sin^2 \left(\frac{\theta_t}{2}\right) = 2 ||S\mu_*|| (1-\cos \theta_t).
\end{flalign*}
On the other hand, using \eqref{eq:thetat}, we obtain
\begin{flalign*}
1- \cos \theta_t  = \frac{||S\mu_* - S\widetilde{\mu}(t)||^2 - (||S\mu_*|| - ||S\widetilde{\mu}(t)|| )^2}{2||S\mu_*||~||S\widetilde{\mu}(t)|| } \leq   \frac{||S\mu_* - S\widetilde{\mu}(t)||^2 }{2||S\mu_*||~||S\widetilde{\mu}(t)|| }.
\end{flalign*}
To proceed, define $\eta(t)  = S\widetilde{\mu}(t)-S\mu_* + S\mathbb{E}[B(t)^{-1}] \Sigma^{-1}\mu_*$, and note that $\mathbb{E}[\eta(t)\eta(t)^T] = S\mathrm{Cov}(\widetilde{\mu}(t))S$. So, it holds that
\begin{flalign*}
\mathbb{E}[1-\cos \theta_t] 
\leq \mathbb{E}\left[\frac{||\eta(t) - S\mathbb{E}[B(t)^{-1}] \Sigma^{-1}\mu_*||^2}{2||S\mu_*||~||S\widetilde{\mu}(t)|| }\right] \leq \mathbb{E}\left[\frac{||\eta(t)||^2 + ||S\mathbb{E}[B(t)^{-1}] \Sigma^{-1}\mu_*||^2 }{||S\mu_*||~||S\widetilde{\mu}(t)|| }\right].
\end{flalign*}
By Corollary \ref{col:1}, we have 
\begin{eqnarray}
\mathbb{E}[||\eta(t)||^2] =  \mathbf{tr}\left(\mathbb{E}[\eta(t)\eta(t)^T]\right)
=  \mathbf{tr}(S\mathrm{Cov}(\widetilde{\mu}(t))S) = O(dt^{-1}).\label{eq:eta}
\end{eqnarray}
Accordingly, we get
\begin{flalign*}
\mathbb{E}\left[\frac{||\eta(t)||^2 + ||S\mathbb{E}[B(t)^{-1}] \Sigma^{-1}\mu_*||^2 }{||S\mu_*||~||S\widetilde{\mu}(t)|| }\right] = O(d t^{-1}), 
\end{flalign*}
because the expected value of the numerator is $O(t^{-1})$ by \eqref{eq:eta} and Theorem \ref{thm:1}, while the denominator converges to $||S \mu_*||^2$ as $t\rightarrow \infty$, by Corollary \ref{col:1}. Thus, we have 
\begin{eqnarray}
\sum_{t=1}^T \mathbb{E}\left[\frac{||\eta(t)||^2 + ||S\mathbb{E}[B(t)^{-1}] \Sigma^{-1}\mu_*||^2 }{||S\mu_*||~||S\widetilde{\mu}(t)|| }\right]
= O(d\log T).\label{eq:log}
\end{eqnarray}
Putting the latter result together with \eqref{eq:cnlogn}, it yields to the desired result, since $c_N$ depends only on $N$, and $||S\mu_*||$ is a constant:
\begin{flalign*}
\mathrm{Regret}(T) = \sum_{t=1}^T c_N \mathbb{E}[2 ||S\mu_*|| (1-\cos \theta_t)]  = O(d \sqrt{\log N} \log T).
\end{flalign*}
\end{proof}

\section{Numerical Illustrations}
\label{sec:5}
\begin{figure*}[h]
    \psfrag{error~~~~~~~~~~~~~~~~~~~~~~~~~~}{\scriptsize$\mathbb{E}\left[||\widehat{\mu}(t)-\mu_*||/\sqrt{d}\right]$}
    \psfrag{d=10}{\scriptsize$\mathbf{d=10}$}
    \psfrag{d=30}{\scriptsize$\mathbf{d=30}$}
    \psfrag{t}{\scriptsize t}
    \psfrag{N=5~~~}{\scriptsize $N=5$}
    \psfrag{N=10~~~}{\scriptsize $N=10$}
    \psfrag{N=20~~~}{\scriptsize $N=20$}
    \psfrag{N=50~~~}{\scriptsize $N=50$}
    \psfrag{Regret~~~~~~~~~~~~~~~}{\scriptsize $\mathbf{Regret}(t)/\left(d \log t \sqrt{\log N}\right)$}
    \centering
    \includegraphics[width=0.49\textwidth]{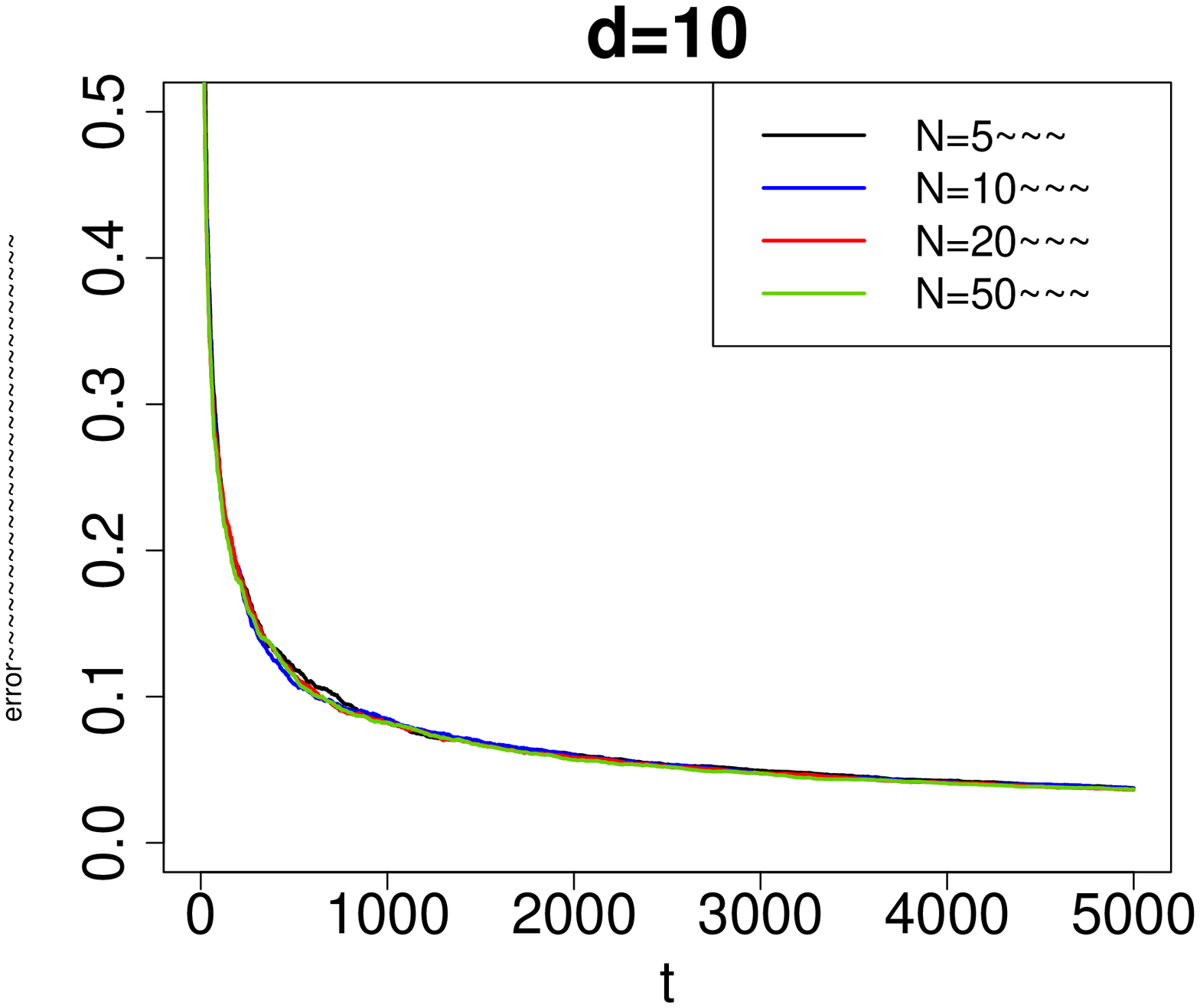}
    \includegraphics[width=0.49\textwidth]{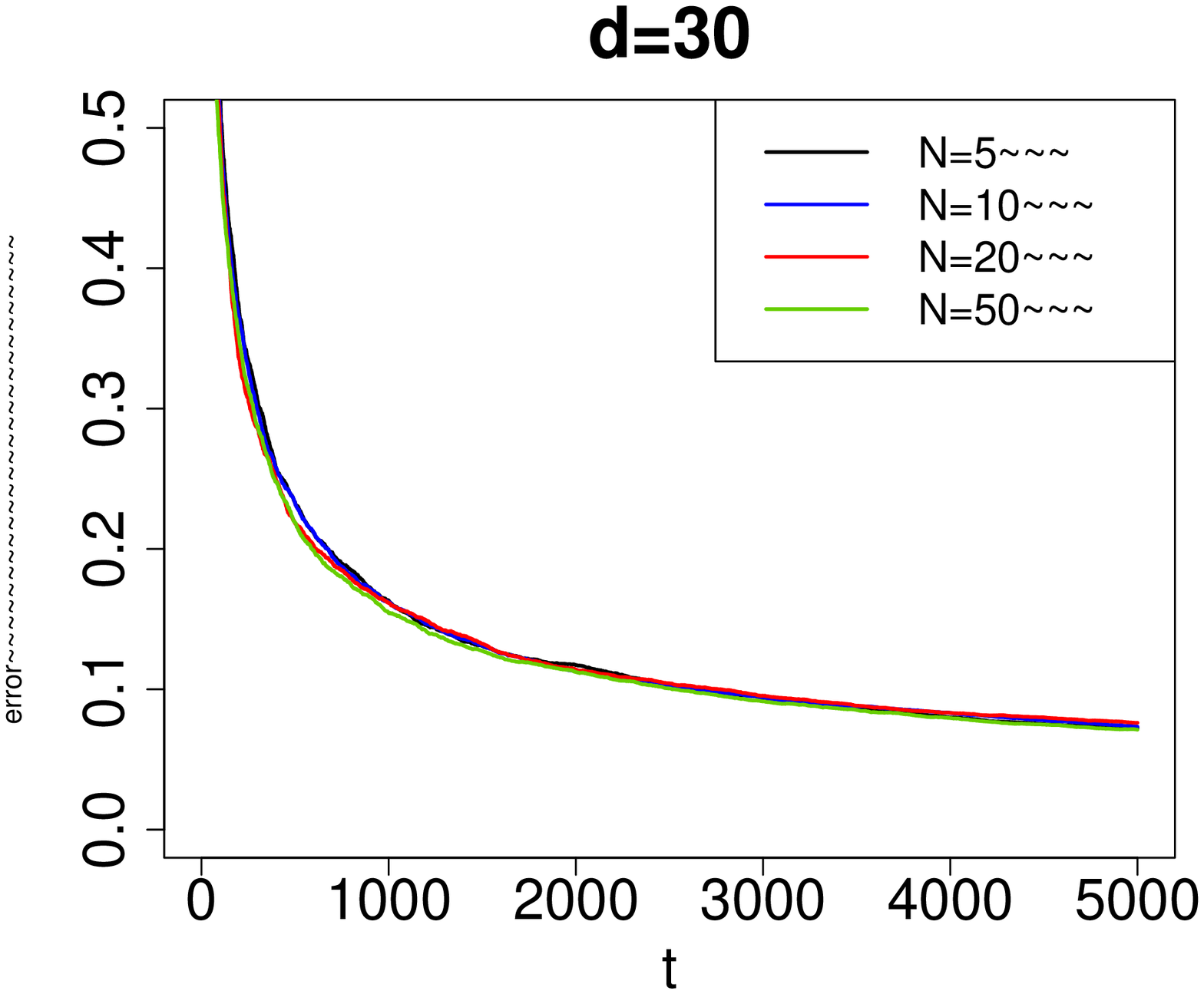}
    \caption{Plots of $\mathbb{E}\left[||\widehat{\mu}(t)-\mu_*||/\sqrt{d}\right]$ over time for different number of arms $N=5,10,20,50$, and different dimensions of the contexts $d=10,30$.}
    \label{fig:1}
\end{figure*}
\begin{figure*}[h]
    \psfrag{error~~~~~~~~~~~~~~~~~~~~~~~~~~}{\scriptsize$\mathbb{E}\left[||\widehat{\mu}(t)-\mu_*||/\sqrt{d}\right]$}
    \psfrag{d=10}{\scriptsize$\mathbf{d=10}$}
    \psfrag{d=30}{\scriptsize$\mathbf{d=30}$}
    \psfrag{t}{\scriptsize t}
    \psfrag{N=5~~~}{\scriptsize $N=5$}
    \psfrag{N=10~~~}{\scriptsize $N=10$}
    \psfrag{N=20~~~}{\scriptsize $N=20$}
    \psfrag{N=50~~~}{\scriptsize $N=50$}
    \psfrag{Regret~~~~~~~~~~~~~~~}{\scriptsize $\mathbf{Regret}(t)/\left(d \log t \sqrt{\log N}\right)$}
    \centering
    \includegraphics[width=0.49\textwidth]{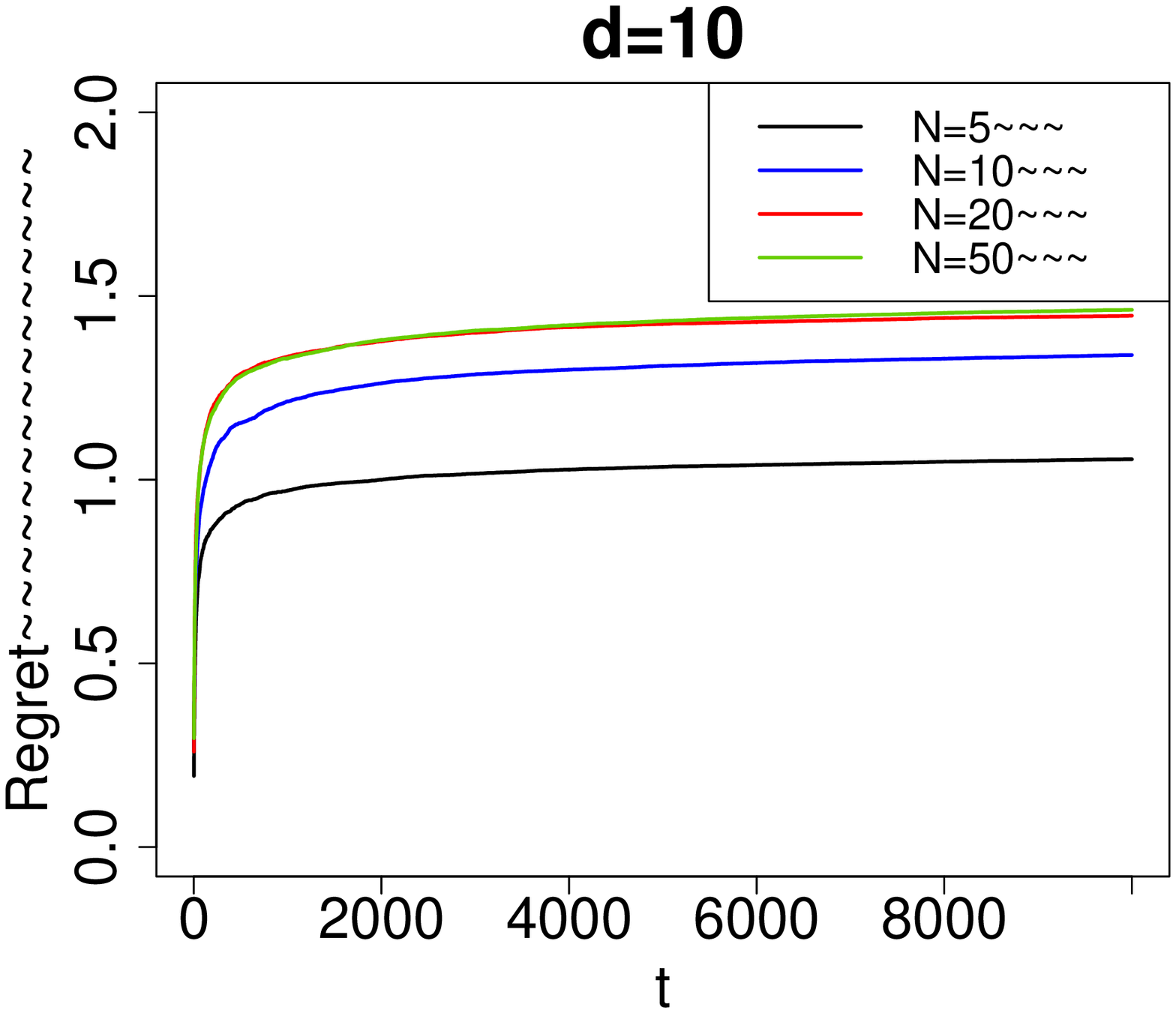}
    \includegraphics[width=0.49\textwidth]{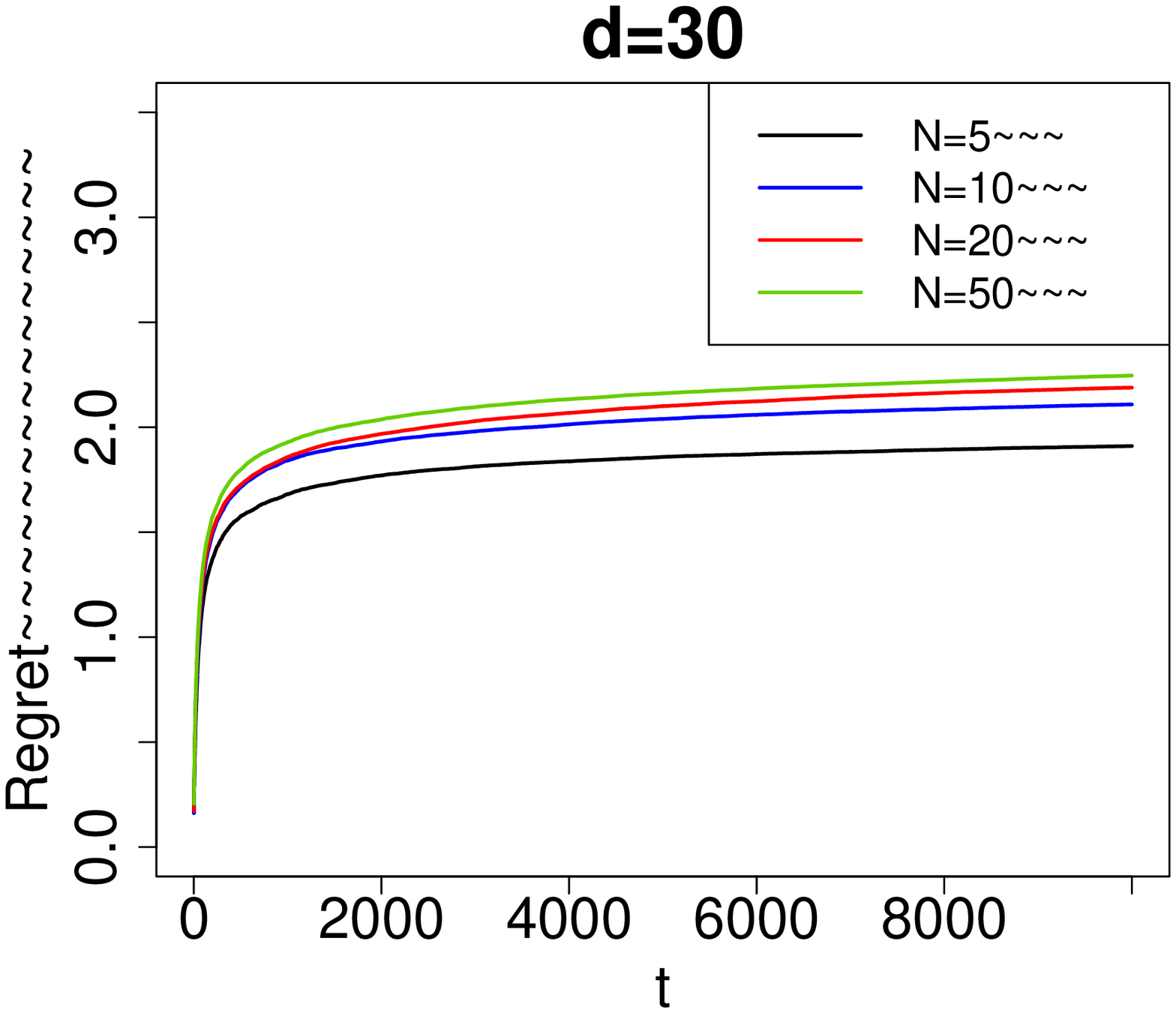}
    \caption{Plots of the regret normalized by $d \log t \sqrt{\log N}$, over time for different number of arms $N=5,10,20,50$, and the dimension of the context $d=10,30$.}
    \label{fig:2}
\end{figure*}

We consider cases with different numbers of arms, $N=5,~10,~20,~50,$ and different dimensions of the contexts $d=10,~30$, repeating 50 times for each case, for every time step. We report two quantities, $||\widehat{\mu}(t) - \mu_*||$ and $\mathrm{Regret}(t)$, over time, and take averages of the quantities for 50 scenarios. The true parameter $\mu_*$ as well as each row of $A$, are randomly generated. Further, we let $\Sigma_x = I_d$, $\Sigma_y = I_d$, and $\sigma^2 = 1$.

Figure \ref{fig:1} depicts the average norm of the normalized errors over time. We normalize the errors by $\sqrt{d}$, since $\mathrm{Cov}(\widetilde{\mu}(t))=O(t^{-1})$, by Corollary \ref{col:1}, and so $\mathbf{tr}(\mathrm{Cov}(\widetilde{\mu}(t))) = O(dt^{-1})$. The curves in Figure \ref{fig:1} show that the errors decrease with the appropriate rates. Figure \ref{fig:2} illustrates the normalized regret over time. The regret is normalized by its bound $d\log t \sqrt{\log N}$ in Theorem \ref{thm:1}. In Figure \ref{fig:2}, the curves show that the normalized regret is constant over time, corroborating the regret bound in Theorem \ref{thm:2}.

\section{Concluding remarks}
\label{sec:6}
We studied the design and analysis of a reinforcement learning policy for partially observable contextual multi-armed bandits. First, we presented a modified version of Thompson Sampling that leverages Bayesian methods for balancing the exploration and the exploitation, and estimates the unobserved contexts based on the sequence of output observations. Further, we show that the parameter estimates converge fast to the truth, and that as time goes by, the presented algorithm learns the unknown true parameter accurately.  Finally, we established theoretical performance guarantees showing that the regret of the proposed algorithm scales linearly with dimension, and logarithmically with time and the number of arms.

Extending the presented framework to similar reinforcement learning problems is of interest, including partially observed (contextual) Markov decision processes.  Moreover, addressing the problem when the transformation matrix is unknown, and incorporating an estimation procedure for that, is an interesting direction for future work. Finally, settings with large-scale action spaces and those with high-dimensional parameters constitute further topics for future studies.

\ifarxiv
\appendices
\input{appendixA}
\input{appendixB}
\input{Bibliofile}
\else
\bibliographystyle{IEEEtran}    
\fi 
\thispagestyle{empty}
\bibliography{mybib}

\section*{Appendix}

\subsection*{Derivation of the conditional distribution $\mathbb{P}(x_i(t)|y_i(t))$}

Note that $y_i(t) = Ax_i(t) + \varepsilon_{yi}(t)$, where the distributions of $\varepsilon_{y_i}(t)$ and $x_i(t)$ are $N(0_d, \Sigma_y)$ and $N(0_d, \Sigma_x)$, respectively. The conditional distribution of $x_i(t)$ given $y_i(t)$ can be calculated as follows.
\begin{flalign}
\mathbb{P}(x_i(t)|y_i(t)) &\propto \mathbb{P}(y_i(t)|x_i(t)) \mathbb{P}(x_i(t))\nonumber\\
&\propto \exp\left((y_i(t)-Ax_i(t))^\top\Sigma_y^{-1}(y_i(t)-Ax_i(t)) \right) \exp\left(x_i(t)^\top \Sigma_x^{-1}x_i(t)\right)\nonumber\\
&\propto N((A^{\top} \Sigma_y^{-1} A + \Sigma_x^{-1})^{-1}A^\top \Sigma_y^{-1} y_{i}(t),(A^{\top} \Sigma_y^{-1} A + \Sigma_x^{-1})^{-1})
\end{flalign}

\subsection*{Derivation of the conditional distribution $\mathbb{P}(r_i(t)|y_i(t))$}

Let $\Sigma_{xy}= (A^{\top} \Sigma_y^{-1} A + \Sigma_x^{-1})^{-1}$ and  recall $\widehat{x}_i(t) = (A^{\top} \Sigma_y^{-1} A + \Sigma_x^{-1})^{-1}A^\top \Sigma_y^{-1} y_{i}(t) = Dy_{i}(t)$.
\begin{flalign}
\mathbb{P}(r_i(t)|\mu,y_i(t)) 
&= \int_{\mathbb{R}^d } \mathbb{P}(r_i(t)|\mu,x_i(t)) \mathbb{P}(x_i(t)|y_i(t)) dx_i(t)\nonumber\\
&\propto \int_{\mathbb{R}^d } \exp\left( -\frac{(r_i(t)-x_i(t)^\top \mu)^2}{2\sigma^2} \right) \exp\left( - \frac{1}{2}(x_i(t)-\widehat{x}_i(t) )^\top \Sigma_{xy}^{-1} (x_i(t)-\widehat{x}_i(t) ) \right) dx_i(t)\nonumber\\
&\propto \exp\left(- \frac{\left(r_i(t) - ((A^{\top} \Sigma_y^{-1} A + \Sigma_x^{-1})^{-1}A^\top \Sigma_y^{-1} y_i(t))^\top \mu \right)^2}{2(\mu^\top \Sigma_{xy} \mu + \sigma^2)}    \right)\nonumber\\
&\propto N\left(  \widehat{x}_i(t)^\top \mu, \sigma^2_{ry} \right)\label{eq:condry}.
\end{flalign}

\subsection*{Derivation of the posterior 
$\mathbb{P}(\mu|\mathbf{r}_{t-1},\mathbf{y}_{t-1})$ }
Let $\mathbb{P}(\mu)$, the pdf of $N(0,  \sigma^2_{ry} \Sigma )$, be the prior of $\mu_*$. We can decompose the posterior as follows.
\begin{flalign*}
\mathbb{P}(\mu|\mathbf{r}_{t-1},\mathbf{y}_{t-1}) 
&\propto \mathbb{P}(\mathbf{r}_{t-1},\mathbf{y}_{t-1}|\mu)\mathbb{P}(\mu)\\
&\propto \mathbb{P}(\mathbf{r}_{t-1}|\mathbf{y}_{t-1},\mu)\mathbb{P}(\mu).
\end{flalign*}

Using the prior and the conditional distribution in \eqref{eq:condry}, we have

\begin{flalign}
\mathbb{P}(\mu|\mathbf{r}_{t-1},\mathbf{y}_{t-1}) &\propto \prod_{\tau=1}^{t-1} \exp\left(-\frac{ (r_{a(\tau)}(\tau) - \widehat{x}_{a(\tau)}(\tau)^\top \mu)^2}{2\sigma^2_{ry}} \right) \exp\left( -\frac{1}{2\sigma^2_{ry} } \mu^\top \Sigma^{-1} \mu \right) \nonumber\\
&\propto \exp\left( -\frac{1}{2\sigma^2_{ry}} \left(\mu- \widehat{\mu}(t) \right)^\top B(t) \left(\mu- \widehat{\mu}(t) \right)  \right),\label{eq:posmu}
\end{flalign}
which is the kernel of the pdf of $N(\widehat{\mu}(t), \sigma^2_{ry} B(t)^{-1})$, where $\widehat{\mu}(t) = B(t)^{-1}\sum_{\tau=1}^{t-1} \widehat{x}_{a(t)}(t) r_{a(t)}(t)$ and $$B(t) = \sum_{\tau=1}^{t-1} \widehat{x}_{a(\tau)}(\tau) \widehat{x}_{a(\tau)}^\top (\tau) + \Sigma^{-1}. $$

Thus, the posterior distribution is $N(\widehat{\mu}(t),\sigma^2_{ry} B(t)^{-1})$. But, to allow for the possibility that 
$\sigma^2_{ry}$ is unknown, we use a re-scaled posterior distribution, $N(\widehat{\mu}(t),B(t)^{-1})$, which does not depend on $\sigma^2_{ry}$.

\subsection*{Derivation of the recursion formula to update the parameter. }

Note that we can decompose the posterior as follows.
\begin{eqnarray}
\mathbb{P}(\mu|\mathbf{r}_{t},\mathbf{y}_{t}) &\propto& \mathbb{P}(\mathbf{r}_{t},\mathbf{y}_{t},\mu) \nonumber\\
&\propto& \mathbb{P}(r_{a(t)}(t) | y_{a(t)}(t),\mu) \mathbb{P}(\mu|\mathbf{r}_{t-1},\mathbf{y}_{t-1}).\nonumber
\end{eqnarray}

Using the conditional distribution \eqref{eq:condry} and the posterior in \eqref{eq:posmu}, we get
\begin{flalign*}
\mathbb{P}(\mu|\mathbf{r}_{t},\mathbf{y}_{t}) &\propto \mathbb{P}(r_{a(t)}(t) | y_{a(t)}(t),\mu) \mathbb{P}(\mu|\mathbf{r}_{t-1},\mathbf{y}_{t-1})\\
&\propto \exp\left(-\frac{ (r_{a(t)}(t) - \widehat{x}_{a(t)}(t)^\top \mu)^2}{2\sigma^2_{ry}} \right) \exp\left( -\frac{1}{2\sigma^2_{ry}} \left(\mu- \widehat{\mu}(t) \right)^\top B(t)^{-1} \left(\mu- \widehat{\mu}(t) \right)  \right)\\
&\propto \exp\left( -\frac{1}{2\sigma^2_{ry}} \left(\mu- \widehat{\mu}(t+1) \right)^\top B(t+1)^{-1} \left(\mu- \widehat{\mu}(t+1) \right)  \right),
\end{flalign*}
where $\widehat{\mu}(t+1) = B(t+1)^{-1} \left(B(t) \widehat{\mu}(t) + \widehat{x}_{a(t)}(t) r_{a(t)}(t)\right)$ and $B(t+1) = B(t) + \widehat{x}_{a(t)}(t)\widehat{x}_{a(t)}(t)^\top$. 

\subsection*{Proof of Lemma \ref{lem1}}
\begin{proof}
	Recall that we used the notation $S= \mathrm{Var} (\widehat{x}_{i}(t))^{0.5} = (D\Sigma_yD^\top)^{0.5}$ and $Z(\mu,N) = \argmax_{Z_i,1 \leq i \leq N} \{Z_i^\top \mu\}$. Note that $S^{-1}\widehat{x}_{i}(t)$ has the distribution $N(0_d, I_d)$ and $S^{-1}\widehat{x}_{a(t)}(t) = Z(S\widetilde{\mu}(t) ,N)$.  $S^{-1}\widehat{x}_{i}(t)$ can be decomposed as $$S^{-1}\widehat{x}_{i}(t)= P_{S\widetilde{\mu}(t)}S^{-1}\widehat{x}_{i}(t) + P_{S\widetilde{\mu}(t)^\perp}S^{-1}\widehat{x}_{i}(t),$$ where $P_{S\widetilde{\mu}(t)^\perp}$ denotes the projection matrix onto a subspace orthogonal to the column-space $C(S\widetilde{\mu}(t))$, which we denote $C(S\widetilde{\mu}(t))^\perp$. As shown in \eqref{eq:equaldis}, we have 
	$$S^{-1} \widehat{x}_{a(t)}(t)   \overset{d}{=}P_{S\widetilde{\mu}(t)} S^{-1} \widehat{x}_{a(t)}(t) + P_{S\widetilde{\mu}(t)^\perp} S^{-1} \widehat{x}_{i}(t),$$ 
	where $\overset{d}{=}$ expresses that the two quantities have an identical distribution. Further, based on the fact that the function $Z(\mu,N)$ defined in \eqref{eq:zmun} is affected only by $\{P_{\mu}Z_i\}_{1\leq i \leq N}$, but not by $\{(I_d-P_{\mu})Z_i\}_{1\leq i \leq N}$, we established that  $P_{S\widetilde{\mu}(t)} S^{-1} \widehat{x}_{a(t)}(t)$ and $P_{S\widetilde{\mu}(t)^\perp} S^{-1} \widehat{x}_{i}(t)$ are statistically independent. Now, consider the following decomposition.
	\begin{flalign*}
	&\mathbb{E}[ S^{-1} \widehat{x}_{a(t)}(t)\widehat{x}_{a(t)}(t)^{\top} S^{-1}|\widetilde{\mu}(t)]\nonumber\\ 
	&= P_{S\widetilde{\mu}(t)} E[ S^{-1} \widehat{x}_{a(t)}(t)\widehat{x}_{a(t)}(t)^{\top} S^{-1}]P_{S\widetilde{\mu}(t)}
	+ P_{S\widetilde{\mu}(t)^\perp}\mathbb{E}[ S^{-1} \widehat{x}_{a(t)}(t)\widehat{x}_{a(t)}(t)^{\top} S^{-1}|\widetilde{\mu}(t)]P_{S\widetilde{\mu}(t)^\perp} \nonumber\\
	&+ P_{S\widetilde{\mu}(t)}\mathbb{E}[ S^{-1} \widehat{x}_{a(t)}(t)\widehat{x}_{a(t)}(t)^{\top} S^{-1}|\widetilde{\mu}(t)]P_{S\widetilde{\mu}(t)^\perp} + P_{S\widetilde{\mu}(t)^\perp}\mathbb{E}[ S^{-1} \widehat{x}_{a(t)}(t)\widehat{x}_{a(t)}(t)^{\top} S^{-1}|\widetilde{\mu}(t)]P_{S\widetilde{\mu}(t)}. \nonumber
	\end{flalign*}
	By replacing $P_{S\widetilde{\mu}(t)} S^{-1} \widehat{x}_{a(t)}(t)$ with $P_{S\widetilde{\mu}(t)^\perp}S^{-1}\widehat{x}_{i}(t)$ based on the independence and the equivalence of the distribution, we get
	\begin{eqnarray}
	P_{S\widetilde{\mu}(t)^\perp}\mathbb{E}[ S^{-1} \widehat{x}_{a(t)}(t)\widehat{x}_{a(t)}(t)^{\top} S^{-1}|\widetilde{\mu}(t)]P_{S\widetilde{\mu}(t)^\perp}
	= P_{S\widetilde{\mu}(t)^\perp}\mathbb{E}[ S^{-1} \widehat{x}_{i}(t)\widehat{x}_{i}(t)^{\top}S^{-1}|\widetilde{\mu}(t)]P_{S\widetilde{\mu}(t)^\perp} 
	= P_{S\widetilde{\mu}(t)^\perp}, \label{eq:dec1}
	\end{eqnarray}
	and
	\begin{eqnarray}
	&&P_{S\widetilde{\mu}(t)}\mathbb{E}[ S^{-1} \widehat{x}_{a(t)}(t)\widehat{x}_{i}(t)^{\top} S^{-1}|\widetilde{\mu}(t)]P_{S\widetilde{\mu}(t)^\perp} + P_{S\widetilde{\mu}(t)^\perp}\mathbb{E}[ S^{-1} \widehat{x}_{i}(t)\widehat{x}_{a(t)}(t)^{\top} S^{-1}|\widetilde{\mu}(t)]P_{S\widetilde{\mu}(t)}\nonumber\\
	&=& P_{S\widetilde{\mu}(t)}\mathbb{E}[ S^{-1} \widehat{x}_{a(t)}(t)|\widetilde{\mu}(t)]\mathbb{E}[\widehat{x}_{i}(t)^{\top} S^{-1}|\widetilde{\mu}(t)]P_{S\widetilde{\mu}(t)^\perp} + P_{S\widetilde{\mu}(t)^\perp}\mathbb{E}[ S^{-1} \widehat{x}_{i}(t)|\widetilde{\mu}(t)]\mathbb{E}[\widehat{x}_{a(t)}(t)^{\top} S^{-1}|\widetilde{\mu}(t)]P_{S\widetilde{\mu}(t)}\nonumber\\
	&=& 0, \label{eq:dec2}
	\end{eqnarray}
	because $\mathbb{E}[\widehat{x}_{i}(t)|\widetilde{\mu}(t)]=0$. Thus, by putting \eqref{eq:dec1} and \eqref{eq:dec2} together, we have
	\begin{flalign*}
	&\mathbb{E}[ S^{-1} \widehat{x}_{a(t)}(t)\widehat{x}_{a(t)}(t)^{\top} S^{-1}|\widetilde{\mu}(t)]= P_{S\widetilde{\mu}(t)} \mathbb{E}[ S^{-1} \widehat{x}_{a(t)}(t)\widehat{x}_{a(t)}(t)^{\top} S^{-1}|\widetilde{\mu}(t)]P_{S\widetilde{\mu}(t)} + P_{S\widetilde{\mu}(t)^\perp}. 
	\end{flalign*}
	
	On the other hand, $P_{S\widetilde{\mu}(t)} \mathbb{E}[ S^{-1} \widehat{x}_{a(t)}(t)\widehat{x}_{a(t)}(t)^{\top} S^{-1}|\widetilde{\mu}(t)]P_{S\widetilde{\mu}(t)}$ can be written as
	\begin{eqnarray}
P_{S\widetilde{\mu}(t)}\mathbb{E}[ S^{-1} \widehat{x}_{a(t)}(t)\widehat{x}_{a(t)}(t)^{\top} S^{-1}|\widetilde{\mu}(t) ]P_{S\widetilde{\mu}(t)}
	&=& \frac{S\widetilde{\mu}(t)\widetilde{\mu}(t)^{\top}S}{\widetilde{\mu}(t)^{\top}S^2\widetilde{\mu}(t)} \mathbb{E}[ S^{-1} \widehat{x}_{a(t)}(t)\widehat{x}_{a(t)}(t)^{\top} S^{-1}|\widetilde{\mu}(t) ] \frac{S\widetilde{\mu}(t)\widetilde{\mu}(t)^{\top}S}{\widetilde{\mu}(t)^{\top}S^2\widetilde{\mu}(t)} \nonumber\\
	&=& \frac{S\widetilde{\mu}(t)}{\widetilde{\mu}(t)^{\top} S^2\widetilde{\mu}(t)} \mathbb{E}[ (\widetilde{\mu}(t)^{\top}  S S^{-1} \widehat{x}_{a(t)}(t))^2|\widetilde{\mu}(t) ]\frac{\widetilde{\mu}(t)^{\top}S}{\widetilde{\mu}(t)^{\top}S^2\widetilde{\mu}(t)} \nonumber\\
	&=& P_{S\widetilde{\mu}(t)} \mathbb{E}\left[ \left. \left( \left(S^{-1} \widehat{x}_{a(t)}(t)\right)^{\top}\overrightarrow{S\widetilde{\mu}(t)}   \right)^2\right|\widetilde{\mu}(t) \right].
	\end{eqnarray}
    Since $\widehat{x}_{i}(t)^{\top} S^{-1}\overrightarrow{S\widetilde{\mu}(t)}$ has a standard normal distribution, we have
	\begin{eqnarray}
	 \mathbb{E}\left[\left. \left(\widehat{x}_{a(t)}^{\top}(t)S^{-1}\overrightarrow{S\widetilde{\mu}(t)} \right)^2\right|\widetilde{\mu}(t)\right] = \mathbb{E}\left[ \left(\underset{1 \leq i \leq N}{max} (\{V_i:V_i \sim N(0,1) \}\right)^2 \right]\label{eq:kn}.
	\end{eqnarray}
	We define the quantity in \eqref{eq:kn} as $k_N$, 
	\begin{eqnarray}
	k_N = \mathbb{E}\left[ \left(\underset{1 \leq i \leq N}{max} (\{V_i:V_i \sim N(0,1) \}\right)^2 \right],
	\end{eqnarray}
	which is greater than 1 and grows as $N$ gets larger, because $\mathbb{E}[V_i^2] = 1 < \mathbb{E}\left[ \left(\underset{1 \leq i \leq N}{max} (\{V_i:V_i \sim N(0,1) \}\right)^2 \right]$.
	Thus, $\mathbb{E}[ S^{-1} \widehat{x}_{a(t)}(t)\widehat{x}_{a(t)}(t)^{\top} S^{-1}|\widetilde{\mu}(t) ]$ can be written as
	\begin{eqnarray}
	\mathbb{E}[ S^{-1} \widehat{x}_{a(t)}(t)\widehat{x}_{a(t)}(t)^{\top} S^{-1}|\widetilde{\mu}(t) ] = P_{S\widetilde{\mu}(t)} k_N + P_{S\widetilde{\mu}(t)^\perp} = P_{S\widetilde{\mu}(t)} (k_N-1) + I_d.\label{eq:matdec}
	\end{eqnarray}
	
	Because the column-spaces of the matrices  $P_{S\widetilde{\mu}(t)}$ and $P_{S\widetilde{\mu}(t)^\perp}$ are orthogonal, the non-zero eigenvalues of $P_{S\widetilde{\mu}(t)} \mathbb{E}[ S^{-1} \widehat{x}_{a(t)}(t)\widehat{x}_{a(t)}(t)^{\top} S^{-1}|\widetilde{\mu}(t)]P_{S\widetilde{\mu}(t)}$ and $P_{S\widetilde{\mu}(t)^\perp}$ are the eigenvalues of $\mathbb{E}[ S^{-1} \widehat{x}_{a(t)}(t)\widehat{x}_{a(t)}(t)^{\top} S^{-1}|\widetilde{\mu}(t)]$. That is, $(d-1)$ eigenvalues of $\mathbb{E}[ S^{-1} \widehat{x}_{a(t)}(t)\widehat{x}_{a(t)}(t)^{\top} S^{-1}|\widetilde{\mu}(t)]$ are $1$, and the other eigenvalue is $k_N$. This means that $\mathbb{E}[ S^{-1} \widehat{x}_{a(t)}(t)\widehat{x}_{a(t)}(t)^{\top} S^{-1}|\widetilde{\mu}(t) ]$ is positive definite, since $k_N>1$.
	
	Next, for the true parameter $\mu_*$, we claim $\lim_{t \rightarrow \infty} \widetilde{\mu}(t) = \mu_*$. With \eqref{eq:emuhat}, \eqref{eq:covmuhat}, and the fact that $\widetilde{\mu}(t)$ is generated from the posterior $N(\widehat{\mu}(t),B(t)^{-1})$, we have
\begin{eqnarray}
\mathbb{E}\left[\widetilde{\mu}(t)\right] &=& \mathbb{E}\left[\mathbb{E}[\widetilde{\mu}(t)|\mathscr{F}_{t-1}]\right] = \mathbb{E}\left[\widehat{\mu}(t)\right]
=(I_d - \mathbb{E}[B(t)^{-1}] \Sigma^{-1})\mu_*,\label{eq:mutilde}\\
\mathrm{Cov}(\widetilde{\mu}(t)) &=&\mathrm{Cov}(\mathbb{E}[\widetilde{\mu}(t)|\mathscr{F}_{t-1}]) + \mathbb{E}[\mathrm{Cov} (\widetilde{\mu}(t)|\mathscr{F}_{t-1})]\nonumber\\
&=&\mathrm{Cov}(\widehat{\mu}(t)) + \mathbb{E}[B(t)^{-1}]\nonumber\\
&=& \mathbb{E}\left[B(t)^{-1} \Sigma^{-1} \mu_*\mu_*^\top \Sigma^{-1} B(t)^{-1} \right]- \mathbb{E} \left[B(t)^{-1} \right] \Sigma^{-1}\mu_*\mu_*^\top \Sigma^{-1}\mathbb{E} \left[B(t)^{-1} \right] \nonumber\\
&+& \mathbb{E}\left[B(t)^{-1}\right]\sigma^2_{ry}- \mathbb{E}\left[B(t)^{-1}\Sigma^{-1} B(t)^{-1}\right]\sigma^2_{ry} +  \mathbb{E}\left[B(t)^{-1}\right] \nonumber\\
&=&\mathbb{E}\left[B(t)^{-1} \Sigma^{-1} \mu_*\mu_*^\top \Sigma^{-1} B(t)^{-1} \right]
- \mathbb{E} \left[B(t)^{-1} \right] \Sigma^{-1}\mu_*\mu_*^\top \Sigma^{-1}\mathbb{E} \left[B(t)^{-1} \right] \nonumber\\
&+& \mathbb{E}\left[B(t)^{-1}\right](\sigma^2_{ry}+1) - \mathbb{E}\left[B(t)^{-1}\Sigma^{-1} B(t)^{-1}\right]\sigma^2_{ry}. \label{eq:covmutil}
\end{eqnarray}

Since $\lim_{t\rightarrow \infty}B(t)^{-1} = 0_{d \times d}$ and thereby $\lim_{t\rightarrow \infty}\mathrm{Cov}(\widetilde{\mu}(t)) = 0_{d \times d}$, $\widetilde{\mu}(t)$ is a consistent estimator of $\mu_*$. That is, 
\begin{eqnarray}
\lim_{t \rightarrow \infty} \widetilde{\mu}(t) = \mu_*\label{eq:limmutilde}.
\end{eqnarray}

 Thus, $\lim_{t \rightarrow \infty}P_{S\widetilde{\mu}(t)} = P_{S\mu_*}$. Using 
	$$\mathbb{E}[ S^{-1} \widehat{x}_{a(t)}(t)\widehat{x}_{a(t)}(t)^{\top} S^{-1}|\mathscr{F}_{t-1}] = \mathbb{E}[\mathbb{E}[ S^{-1} \widehat{x}_{a(t)}(t)\widehat{x}_{a(t)}(t)^{\top} S^{-1}|\widetilde{\mu}(t)]|\mathscr{F}_{t-1}]$$ and \eqref{eq:matdec}, we get
	 \begin{eqnarray}
	 \lim_{t\rightarrow \infty} \mathbb{E}[ S^{-1} \widehat{x}_{a(t)}(t)\widehat{x}_{a(t)}(t)^{\top} S^{-1}|\mathscr{F}_{t-1}] = \lim_{t\rightarrow \infty} \mathbb{E}[ P_{S\widetilde{\mu}(t)}(k_N-1) + I_d |\mathscr{F}_{t-1}] = P_{S\mu_*}(k_N-1) + I_d. \nonumber
	 \end{eqnarray}
	 
	 Because the eigenvalues $P_{S\mu_*}(k_N-1) + I_d$ are $(d-1)$ $1$s and $k_N$, which is greater than 1, $P_{S\mu_*}(k_N-1) + I_d$ is positive definite. Therefore, $\lim_{t \rightarrow \infty} \mathbb{E}[ S^{-1} \widehat{x}_{a(t)}(t)\widehat{x}_{a(t)}(t)^{\top} S^{-1}|\mathscr{F}_{t-1}]$ is positive definite.

\end{proof}

\subsection*{Proof of Corollary 1}

\begin{proof}

Recall $\mathrm{Cov}(\widetilde{\mu}(t))$ in \eqref{eq:covmutil}
\begin{eqnarray}
\mathrm{Cov}(\widetilde{\mu}(t)) &=&\mathbb{E}\left[B(t)^{-1} \Sigma^{-1} \mu_*\mu_*^\top \Sigma^{-1} B(t)^{-1} \right]
- \mathbb{E} \left[B(t)^{-1} \right] \Sigma^{-1}\mu_*\mu_*^\top \Sigma^{-1}\mathbb{E} \left[B(t)^{-1} \right] \nonumber\\
&+& \mathbb{E}\left[B(t)^{-1}\right](\sigma^2_{ry}+1) - \mathbb{E}\left[B(t)^{-1}\Sigma^{-1} B(t)^{-1}\right]\sigma^2_{ry}\nonumber.
\end{eqnarray}

Since $B(t)^{-1} = O(t^{-1})$ by Lemma \ref{lem1} and the other terms are negligible except $\mathbb{E}\left[B(t)^{-1}\right](\sigma^2_{ry}+1)$ in above terms, we have $\mathrm{Cov}(\widetilde{\mu}(t)) = O(t^{-1})$. In addition, we already showed that $\lim_{t\rightarrow \infty} \widetilde{\mu}(t) = \mu_*$ in \eqref{eq:limmutilde}.Therefore, 
$$\lim_{t\rightarrow \infty} \widetilde{\mu}(t) = \mu_*,~~~~\mathrm{Cov}(\widetilde{\mu}(t)) = O(t^{-1}).$$
\end{proof}

\end{document}